%% file: main.tex
\documentclass[letterpaper, 10 pt, conference]{ieeeconf}  

\IEEEoverridecommandlockouts                              
\usepackage{amsmath, amssymb, amsfonts} 
\usepackage{graphicx, bm, cite}         
\usepackage{algorithm, algorithmic}    
\usepackage{xcolor}                    
\usepackage{soul}
\newtheorem{theorem}{Theorem}

\newtheorem{proposition}{Proposition}

\newtheorem{definition}{Definition}
\newtheorem{remark}{Remark}
\newtheorem{assumption}{Assumption}

\newcommand{\E}{\mathbb{E}}
\newcommand{\R}{\mathbb{R}}

\newcommand{\A}{\mathcal{A}}
\newcommand{\K}{\mathcal{K}}
\newcommand{\W}{\mathcal{W}}
\newcommand{\Z}{\mathcal{Z}}

\newcommand{\bx}{\mathbf{x}}
\newcommand{\bu}{\mathbf{u}}
\newcommand{\bw}{\mathbf{w}}

\newcommand{\genDstro}{\mathcal{D}}
\newcommand{\advDstro}{{\mathcal{D}'}}
\newcommand{\empRisk}{\widehat{R}}

\newcommand{\gaussian}{\mathcal{N}}
\newcommand{\dfn}{\mathrel{:}=}

\newcommand{\cmplx}{\mathcal{C}_\delta^n}

\renewcommand{\Re}{\mathbb{R}}

\DeclareMathOperator*{\minimize}{minimize} 		

\title{\LARGE \bf
Distributionally Robust PAC-Bayesian Control
}

\author{Domagoj Herceg and Duarte Antunes
\thanks{$^{1}$The authors are with the Department of Mechanical Engineering,
        Eindhoven university of Technology, Eindhoven, The Netherlands. \newline
        {email: \tt\small \{d.herceg,d.antunes\}@tue.nl}
        }%
}

\begin{document}
\maketitle
\thispagestyle{empty}
\pagestyle{empty}


\begin{abstract}
We present a distributionally robust PAC-Bayesian framework for certifying the performance of learning-based 
finite-horizon controllers. While existing PAC-Bayes control literature 
typically assumes bounded losses and matching training and deployment distributions, we explicitly address unbounded losses 
and environmental distribution shifts (the sim-to-real gap). We achieve this by drawing on two modern lines 
of research, namely the PAC-Bayes generalization theory and distributionally robust optimization via the type-1 Wasserstein distance. 
By leveraging  the System Level Synthesis (SLS) reparametrization, we derive 
a sub-Gaussian loss proxy and a bound on the performance loss due to distribution shift.
Both are tied directly to the operator norm of the closed-loop map. 
For linear time-invariant systems, this yields a computationally tractable optimization-based framework 
together with high-probability safety certificates for deployment in real-world environments 
that differ from those used in training.
\end{abstract}

\input{text/intro}
\input{text/background}

\input{text/pb2}
\input{text/sls_new}
\input{text/numerical_double_integrator}

\section{Conclusion and Future Work}
In this work, we combined PAC-Bayesian generalization analysis with Wasserstein distributional 
robustness and specialized the resulting framework to finite-horizon LTI control in SLS form. 
The key outcome is an explicit, tractable bound in which both concentration and robustness are 
certified through operator norms of the weighted closed-loop map. 
This yields finite-sample guarantees for randomized feasible controllers under disturbance-distribution shift. 
For future work, we envision extending this approach to model uncertainty, sub-exponential 
losses such as squared norm costs, and 
to learning-based robust model predictive control.

\bibliographystyle{IEEEtran}
\bibliography{dropacb}

\end{document}

%% file: text/intro.tex
\section{INTRODUCTION}
The integration of machine learning into control theory has provided powerful tools for 
synthesizing control policies directly from finite, 
noisy datasets. However, learning-based controllers are notoriously vulnerable to distribution shifts, as they usually assume that 
the data generating distribution used for training  matches the deployment (testing) distribution. This discrepancy, 
also referred to as the sim-to-real gap, can substantially degrade performance when a controller trained under nominal conditions 
is deployed in the real-world, where unmodeled disturbances can act as adversarial perturbations.
Compounding the issue is the finite-sample uncertainty that arises from limited training data.

\par
The PAC-Bayesian~\cite{Catoni2007} framework has emerged as a rigorous method to provide high-probability 
finite-sample generalization guarantees for \emph{randomized} learning algorithms. 
Using PAC-Bayes techniques, the authors in~\cite{DziugaiteR17} have provided the first non-vacuous generalization 
bounds in deep neural networks, which was a major breakthrough and a showcase of the potential of PAC-Bayes methods.

Researchers have also applied PAC-Bayesian methods in control, such as in settings that require generalization 
to unseen environments in robotics~\cite{Majumdar_PACB_generalize_novel_env}. In their extended work,  
they also treat the case of mismatch between training and testing distributions, but only use 
information-theoretic measures of $f$-divergences, effectively inflating the complexity term. 
In addition to assuming a hard cap on the losses, the penalty due to distributional robustness is a simple additive 
constant term that does not depend on the control policy.
More recently, authors in~\cite{BeroujeniPACSNOC,boroujeni2025pac_new} have explored PAC-Bayes 
guarantees in non-linear control by parameterizing the underlying system with an inherently stable parameterization in 
order to optimize and certify the resulting performance. However, as hinted before, these works make a 
standard assumption of a single data generating distribution in training and deployment environments and consider bounded losses.

Distributionally robust (DR) optimization~\cite{shapiro2021lectures} methods have had a much more notable presence in control~\cite{van2015distributionally}, being used in many applications such as model predictive control~\cite{myRAMPC}.
DR approaches have emerged as the principal way to combat the \text{optimizer's curse}~\cite{smith2006optimizers_curse}, 
a phenomenon closely related to overfitting. The main idea is to optimize the system's performance according to 
the worst distribution in an \emph{ambiguity set} centered around a nominal distribution~\cite{OJMO_2022__3__A4_0}. This nominal distribution is often 
an empirical one constructed by drawing a finite number of samples. By resorting to Wasserstein distance, the authors in~\cite{EsfahaniKuhn2018} provide probabilistic guarantees that the true distribution lies in an ambiguity set of the empirical one.

\par 
In this work, we extend the PAC-Bayes \emph{control} framework to handle unbounded losses and distribution shifts. 
In fact, we propose a distributionally robust PAC-Bayesian control framework utilizing the 1-Wasserstein distance.
By leveraging the Kantorovich-Rubinstein duality~\cite{villani2009optimal}, the Wasserstein robustness penalty 
explicitly ties the generalization bound to the controller-dependent Lipschitz constant of the closed-loop system. 
This bound allows us to immunize the system performance against distributional shifts in the controller design phase. 
In fact, our approach establishes a robust PAC-Bayesian bound that explicitly accounts for potential misalignment 
between training and deployment environments. 

Finally, we translate this theoretical distributionally robust PAC-Bayesian bound into a computationally tractable 
algorithm for LTI systems via System Level Synthesis (SLS)~\cite{wang2019_SLS}. This amounts to an effective reparameterization
that enables us to provide concrete bounds on the Lipschitz and sub-Gaussian proxies, both of which depend on the controller 
in our framework. We are now ready to summarize our contributions.

\textbf{Contributions:}
Our main contributions are twofold: first, we introduce a Wasserstein distributionally robust extension for 
PAC-Bayesian control for unbounded Lipschitz loss functions; second, by  specializing the framework to 
finite-horizon LTI control in SLS form, we derive 
explicit controller-dependent certificates for both loss concentration and 
deployment robustness from the same closed-loop map. These certificates lead to a tractable  
posterior optimization problem over feasible controllers and finite-sample guarantees under disturbance shift 
with respect to the training distribution, which does not need to be identified.

\par The remainder of the paper is organized as follows: In Section~\ref{sec:formulation}, we provide the reader with 
the necessary background on the dynamical system we consider, distributional robustness, and we introduce the PAC-Bayes framework. 
In Section~\ref{sec:DRPACBayes}, we state our result on distributionally robust PAC-Bayesian control for 
Wasserstein type-1 distance. Then, in Section~\ref{sec:sls}, we reframe the problem in the SLS framework and specialize our results to LTI systems. We provide valid proxies for both sub-Gaussian concentration and the robustness penalty based on the 
closed-loop map induced by the controller.
In Section~\ref{sec:numerical}, we verify our findings with a numerical example.

\textbf{Notation:} Let $\mathbb{R}^n$ be an $n$ dimensional vector of reals and $\Re^{n \times m}$ 
a matrix of reals for $n,m \in \mathbb{N}.$ For a matrix $M \in \Re^{n \times m}$, the operator
$\operatorname{vec}(M)$ denotes the column-stacking vectorization, i.e., $\operatorname{vec}(M) \in \mathbb{R}^{nm}$.
A cone of positive (semi-)definite matrices
of size $n \times n$ is denoted by $(\mathbb{S}^n_{+})$ $\mathbb{S}^n_{++}$. 
Similarly, let $\mathbb{R}_{+}^n$ be a vector of non-negative reals.
Let $\|\cdot\|$ be a norm on $\Re^n$; unless otherwise stated, this is the Euclidean norm.
A function $f:\mathbb{R}^n \rightarrow \mathbb{R}$ is L-Lipschitz if 
$\| f(x') - f(x)\| \le L \|x' - x\|$ for some $L>0$ and all $x',x$.
The operator norm of the matrix $M \in \Re^{n \times m}$ is defined as 
$\|M\|_{\text{op}} = \sup_{\|w\| = 1} \|Mw\|$ with $w \in \Re^{m}$.
The expectation operator is denoted by $\mathbb{E}$. When needed, 
we indicate the distribution in the subscript, e.g., $\mathbb{E}_{K \sim Q}[f(K)]$
denotes expectation w.r.t. $Q$.

%% file: text/background.tex
\section{Problem formulation}
\label{sec:formulation}
In this section, we introduce the necessary background for our approach.
Our goal is to provide finite-sample guarantees for learning controllers based on PAC-Bayes 
learning theory and to address the distribution
shift between the learned controller in the training and deployment (testing) environments.
To this end, we set up an abstract learning control problem in Section~\ref{sec:general_control_sytems}. 
We introduce the necessary preliminaries in probability theory in Section~\ref{sec:probability}. 
In Section~\ref{sec:dro_wasserstein} we discuss distributional robustness and Wasserstein distance. 
In Section~\ref{sec:PACBayes}, we summarize a key result from~\cite{PAC-Bayes-Chernoff-Bounds-Unbounded-Losses} 
on PAC-Bayes generalization, which we aim to extend to control settings in a distributionally robust form. 
We provide a short problem statement in Section~\ref{sec:problem_statement}.

\subsection{System Dynamics and Control}
\label{sec:general_control_sytems}

We consider a general stochastic control learning setting with controller class
$\mathcal K$ and data space $\mathcal Z$. A hypothesis
$K \in \mathcal K$ represents a controller, while a sample
$Z \in \mathcal Z \subseteq \R^d$ represents a single \emph{data point} of an uncertainty realization affecting the rollout loss,
such as a disturbance trajectory, model uncertainty, etc.

The loss of controller $K$ on sample $Z$ is given by a measurable nonnegative \emph{rollout} loss
$$
\ell : \mathcal K \times \mathcal Z \to \mathbb R_+.
$$
Given a training sample
$$
S = (Z_1,\dots,Z_n) \sim \genDstro^n
$$
and a controller $K\in\mathcal K$, the corresponding population risk~\eqref{eq:population_risk} and
empirical risk~\eqref{eq:empirical_risk} are defined as
\begin{subequations} \label{eq:pac_risks_both}
\begin{align}
R(K):=\mathbb{E}_{Z\sim D}[\ell(K,Z)], \label{eq:population_risk}\\
\empRisk_S(K):=\frac{1}{n}\sum_{i=1}^n \ell(K,Z_i).\label{eq:empirical_risk}
\end{align}
\end{subequations}
In Section~\ref{sec:sls}, we specialize this abstract setup for
finite-horizon LTI systems in System Level Synthesis (SLS) form, where the
sample $Z$ becomes a disturbance trajectory and the rollout loss measures 
weighted closed-loop performance output.

\subsection{Probability}
\label{sec:probability}
The central object of interest in the PAC-Bayes bound is a sub-Gaussian random variable~\cite{boucheron2013concentration}.
\begin{definition}[Sub-Gaussian r.v.] 
A zero mean random variable $X$ is sub-Gaussian if 
\begin{align}
    \E \left[\exp \left( \lambda X \right) \right] \le \exp \left( \frac{\sigma^2 \lambda^2}{2}\right), \quad \text{for any } \lambda \in \R 
\end{align}
for a given constant $\sigma^2$ called the \emph{variance proxy}. Equivalently, such $X$ is also called $\sigma$-sub-Gaussian.
\label{def:sub-Gaussian} 
$\hfill \square$
\end{definition}
Notice that the definition above includes, among others, Gaussian variables where 
Definition~\ref{def:sub-Gaussian} holds with 
equality and almost surely bounded random variables  $a \le X \le b$, where $\sigma^2 = \tfrac{1}{4}(b-a)^2$.
Put another way, sub-Gaussian random variables have a moment generating function (MGF) 
that is uniformly bounded above by the
moment generating function of a Gaussian variable.
\begin{definition}[KL divergence] 
Let $Q,P$ be two probability distributions such that $Q \ll P$ (absolute continuity) 
is defined on a common measurable space $\mathcal{X}$, with densities $q(x),p(x)$. Then
\begin{align}
\label{eq:KL_definition}
\mathrm{KL}(Q \Vert P) 
&= \int_{x \in \mathcal{X}} q(x) \log \frac{q(x)}{p(x)} \, dx.
\end{align}
$\hfill \square$
\end{definition}

The definition for discrete random variables follows analogously.
Here we also note that the KL divergence between two Gaussian random vectors has a closed form solution.
\subsection{Distributional robustness and Wasserstein distance} 
\label{sec:dro_wasserstein}
To quantify the discrepancy between the training and deployment environments, we will employ the 
$1$-Wasserstein distance~\cite{villani2009optimal}. Let $\mathcal{P}(\Z)$ be the set of probability 
measures $P$ on the data space $\Z \subseteq \mathbb{R}^{d}$ with a finite first moment.

\begin{definition}[1-Wasserstein Distance]
For any $P, Q \in \mathcal{P}(\Z)$, the $1$-Wasserstein distance is defined as the minimum cost 
of transporting one distribution to another:
\begin{align}
W_1(P, Q) = \inf_{\pi \in \Pi(P, Q)} \int_{\Z \times \Z} \|z - z'\|_2 \, d\pi(z,z')
\end{align}
where $\Pi(P, Q)$ denotes the set of all joint distributions (couplings) on $\Z \times \Z$ with marginals $P$ and $Q$. $\hfill \square$
\end{definition}
Based on this metric, we define the \textit{Wasserstein ambiguity set} of radius $\rho$ 
centered at distribution $P$ as:
\begin{align}
\mathcal{A}^{W_1}(P) := \{ Q \in \mathcal{P}(\Z) : W_1(Q, P) \le \rho \}
\label{eq:w1_ball}
\end{align}
This set represents a family of distributions that are close to distribution $P$ in an optimal transport sense. 
By Kantorovich–Rubinstein duality\cite{kuhn2019wasserstein} 
for any L-Lipschitz function, $\ell:\Z \mapsto \Re$
\begin{align} 
\sup_{Q \in \mathcal{A}_\rho(P) } \mathbb{E}_{Z \sim Q}[\ell(Z)] \le \mathbb{E}_{P}[\ell(Z)] + \rho L
\end{align} 
The above form is also often referred to as distributionally robust~\cite{shapiro2021lectures} optimization.
However, in our case, the Lipschitz constant and the sub-Gaussian variance proxy 
will not be static objects, they will depend on the posterior controller distribution.

%% file: text/pb2.tex
\subsection{PAC-Bayesian Learning Theory}
\label{sec:PACBayes}
The PAC-Bayesian generalization theory~\cite{Catoni2007} provides high-probability 
generalization guarantees for a \emph{distribution} of hypotheses. This is an important 
distinction from other statistical methods. Instead of bounding the difference between the population and empirical risks~\eqref{eq:pac_risks_both},  
it bounds the $\E_{K \sim Q} [R(K)]$ in terms of $\E_{K \sim Q} [\empRisk_S(K)]$, where $Q$ is a \emph{distribution} over $\K$, 
along with an additional complexity term that we 
explain later on.
These are often called the \emph{population Gibbs} and \emph{empirical Gibbs} risks.
It is a distribution-free approach that assumes we can sample from an unknown distribution and 
provides generalization guarantees through finite sample reasoning.

The majority of PAC-Bayesian results yield a bound on the generalization performance with an assumed maximum cap on the loss.
This setup is natural in machine learning applications, as the loss is often bounded (like $0$-$1$ classification loss).
Translating this to control could be done by saturating the loss at a certain maximum value. However, this entails problems such as tedious calibration and an inevitable loss of sensitivity in the region close to the upper bound.

There are also results for unbounded sub-Gaussian losses~\cite{friendlyPACBayes} 
with an assumed global variance proxy $\sigma^2$ 
for the loss.
However, this is often not suitable for the control of dynamical systems, as  
bounding the worst case loss implies a bound inherently 
determined by the \emph{worst} possible controller among controllers. 

For this reason, we adopt a recent result for \emph{hypothesis dependent sub-Gaussian losses}~\cite{PAC-Bayes-Chernoff-Bounds-Unbounded-Losses}. Here, we briefly note that the approach in~\cite{PAC-Bayes-Chernoff-Bounds-Unbounded-Losses} has other relevant 
implications in terms of optimization, but this is out of the scope of this paper. In control parlance, a hypothesis corresponds to a controller. The meaning is that the concentration bound depends on the (expected) performance of the \emph{deployed} controller. As we seek to optimize the performance of the controller, these two objectives are aligned.
We state the theorem adapted to our setup.

\begin{theorem}[Thm. $11$, Cor. $13$ from~\cite{PAC-Bayes-Chernoff-Bounds-Unbounded-Losses}]
\label{thm:pacbayes_model_dependent_subgaussian_sqrt}
Let $\mathcal{K}$ be a controller (hypothesis) space, let $P$ be a \emph{data-independent} prior on
$\mathcal{K}$, and let
$
S=(Z_1,\dots,Z_n)\sim \genDstro^n,
$ for $n\ge 2$ be a dataset sampled i.i.d. from $\genDstro$.
For each $K\in\mathcal{K}$, let the population and the empirical risk be as in~\eqref{eq:pac_risks_both}.

Assume that for every fixed $K\in\mathcal{K}$, the centered loss
$ 
\widehat{\ell}(K,Z) \dfn \ell(K,Z)-R(K)
$
is $\sigma(K)$-sub-Gaussian, in the sense that
$$
\log \mathbb{E}_{Z\sim D}
\exp\!\left(\lambda(\widehat{\ell}(K,Z)\right)
\le
\frac{\lambda^2}{2}\sigma(K)^2
%
$$
for all $\lambda \in \Re$. 
Then, for any $\delta\in(0,1)$, with probability at least $1-\delta$ over the
draw of $S\sim D^n$, for all posteriors
$Q\ll P$ it holds that 
\begin{align}
\mathbb{E}_{K\sim Q}[R(K)]
\le
\mathbb{E}_{K\sim Q}[\empRisk_S(K)]
+
\cmplx(Q,P,\sigma)
\label{eq:sigma_K_dependend_PAC_Bayes}
\end{align} 
provided that 
$
\mathbb{E}_{K\sim Q}[\sigma(K)^2]<\infty. \hfill\square 
$
\end{theorem} 
In~\eqref{eq:sigma_K_dependend_PAC_Bayes}, shorthand $\cmplx(Q,P,\sigma) $ is the complexity term
$$
\cmplx(Q,P,\sigma) =  \sqrt{
\frac{
2\,\mathbb{E}_{K\sim Q}[\sigma(K)^2]\,
\bigl(\mathrm{KL}(Q\|P)+\log(n/\delta)\bigr)
}{
n-1
}
}.
$$
Note that we sample from the distribution $\genDstro$ used for training and provide the deployment guarantees 
assuming that the data generating distribution in the real-world is the same as the training one. 
This is often an unrealistic assumption, and we will address this particular issue in Section~\ref{sec:DRPACBayes}.
\subsection{Problem statement}
\label{sec:problem_statement}
In this paper, we tackle the problem of ensuring that the control system trained on a finite-sample training environment, 
with an unknown data-generating  distribution, translates to predictable performance in the deployment environment under distribution shift.
In particular, we aim to extend Theorem~\ref{thm:pacbayes_model_dependent_subgaussian_sqrt} to the case where the deployment generating distribution is not the same as the training one.

\section{Distributionally robust PAC-Bayes}
\label{sec:DRPACBayes}
In this section, we combine the nominal PAC-Bayes bound of 
Theorem~\ref{thm:pacbayes_model_dependent_subgaussian_sqrt} with Wasserstein robustification. We further   
specialize the resulting guarantee to finite-horizon SLS control in Section~\ref{sec:sls}.
\par Departing from the standard PAC-Bayes, our aim is to provide a high-probability upper bound on the expected 
\emph{distributionally robust population} (DROP) risk,
which we define as:
\begin{align}
R_\rho(K) \dfn \sup_{\advDstro \in \A_\rho(\genDstro)} \mathbb{E}_{Z \sim \advDstro} [\ell(K, Z)]
\label{eq:DROP_risk}
\end{align}
We indicate the robust version with the $\rho$ in the subscript, which will be the robustness radius throughout the paper.
The main idea is to account for the possibility that the data seen in deployment will be generated 
from a different distribution $\advDstro$, thereby invalidating the assumption of Theorem~\ref{thm:pacbayes_model_dependent_subgaussian_sqrt}.
We treat the real-world distribution as adversarial but close, in 
some sense, to the training regime. The DROP risk stated above is for a generic ambiguity set 
$\A_\rho$ with radius $\rho$. Note that for $\rho = 0$, we recover the standard population risk from Theorem~\ref{thm:pacbayes_model_dependent_subgaussian_sqrt}.
In what follows, we specialize the ambiguity set to type-$1$ Wasserstein ambiguity 
and leave the more general case of arbitrary convex ambiguity sets for future work. 

\subsection{Distributionally Robust PAC-Bayes via Wasserstein Distance}
To derive a bound where the penalty reflects the specific vulnerability of each model, 
we shift our attention to the 1-Wasserstein distance. 
The Wasserstein ambiguity ball of radius $\rho \ge 0$ around the nominal distribution $\genDstro$ is given 
by $A_\rho(\genDstro)$ as defined in~\eqref{eq:w1_ball}.
We shall drop the ${W_1}$ superscript and, from now on, understand that we  are working with the Wasserstein distance 
when referring to ambiguity sets.

\begin{assumption}[Environment]
For any controller $K \in \K$, the loss function $Z \mapsto \ell(K, Z)$ is $L(K)$-Lipschitz 
continuous with respect to the Euclidean metric, 
that is, for all $Z, Z' \in \mathcal{Z}$:
\begin{align*}
|\ell(K, Z) - \ell(K, Z')| \le L(K) \|Z - Z' \|.
\end{align*}
$\hfill \square$
\label{ass:lip_loss}
\end{assumption}
Assumption~\ref{ass:lip_loss} is the regularity condition that enables the Wasserstein robustification in Theorem~\ref{thm:PACBayes-W1}.
We are now ready to state the theorem for the distributionally robust version of PAC-Bayes.
\begin{theorem}[Wasserstein DR-PAC-Bayes]
\label{thm:PACBayes-W1}
Let $\K$ be a hypothesis space, $P$ a data-independent prior, and $S = \{Z_i\}_{i=1}^n$ drawn i.i.d. from $\genDstro$. 
Assume that Assumption~\ref{ass:lip_loss} holds. In addition, for every fixed $K \in \K$, the centered nominal loss is $\sigma(K)$-sub-Gaussian under $\genDstro$.
For any $\delta \in (0, 1)$, with probability at least $1 - \delta$ over the draw of $S \sim \genDstro^n$, the following holds simultaneously for all posteriors $Q \ll P$:
$$
\mathbb{E}_{K \sim Q}  \left[R_\rho(K) \right]\le \mathbb{E}_{K \sim Q}\left[ \widehat{R}_{S, \rho}(K) \right] + \cmplx(Q,P,\sigma)
$$
where $\widehat{R}_{S, \rho}(K) \dfn \widehat{R}_S(K) + L(K)\rho$ and 
$R_\rho(K)$ is given by~\eqref{eq:DROP_risk} with respect to the ambiguity set~\eqref{eq:w1_ball}.
\end{theorem}
\begin{proof}
The idea of the proof is as follows. First, we bound the performance degradation caused by the 
deployment distribution $\advDstro$ with respect to the training data-generating distribution $\genDstro$.
The training data-generating distribution is unknown, but we can sample from it in an i.i.d. fashion and use PAC-Bayes theory 
to provide performance bounds based on a finite-sample dataset. Chaining these two gives the result.

We begin by bounding the worst-case risk for a fixed hypothesis $K \in \K$. Since the loss $Z \mapsto \ell(K,Z)$ is
$L(K)$-Lipschitz by Assumption~\ref{ass:lip_loss}, using the Kantorovich-Rubinstein duality, 
the supremum of the expected loss over the 1-Wasserstein ball is upper bounded by the nominal 
expected loss plus a penalty proportional to the Lipschitz constant of the function. Hence for $K \in \K$
\begin{align}
\sup_{\advDstro \in \A_\rho(\genDstro)} \mathbb{E}_{Z \sim \advDstro} [\ell(K,Z)] \le R(K) + L(K)\rho.
\label{eq:proof_w1_rk}
\end{align}

Taking the expectation over the posterior distribution $Q$ on both sides
\begin{align}
\mathbb{E}_{K \sim Q} \left[ R_\rho(K) \right] \le \mathbb{E}_{K \sim Q} \left[ R(K) + L(K)\rho \right].
\label{eq:proof_robust_pac}
\end{align}
To bound the nominal population risk $R(K)$, we invoke Theorem~\ref{thm:pacbayes_model_dependent_subgaussian_sqrt}
, since by assumption, the centered nominal loss is  $\sigma(K)$-sub-Gaussian under $\genDstro$. Hence,
with probability at least $1 - \delta$ over the sample $S$, it holds simultaneously for all $Q$  that
\begin{align}
\mathbb{E}_{K \sim Q} [R(K)] \le \mathbb{E}_{K \sim Q} [\widehat{R}_S(K)] + \cmplx(Q, P,\sigma).
\label{eq:proof_satndard_pac}
\end{align}
Substituting \eqref{eq:proof_satndard_pac} into \eqref{eq:proof_robust_pac} gives the result.
\end{proof}
Notice that the Wasserstein DROP risk is centered around the training (unknown) distribution.
The robustness penalty $L(K)\rho$ is now explicitly tied to the geometric sensitivity of each individual controller. 
Consequently, minimizing this bound requires an algorithm to actively optimize the empirical robust risk, 
favoring models with small (expected) Lipschitz constants that result in a better bound on 
generalization in unseen environments. 

%% file: text/sls_new.tex
\section{System Level Synthesis and Tractable Posterior Optimization}
\label{sec:sls}

We now specialize the abstract setup of Section~\ref{sec:general_control_sytems}
to finite-horizon linear time-invariant (LTI) systems. 
Controller $K\in\mathcal{K}$ is now a finite-horizon linear causal controller,
and the sample $Z\in\mathcal{Z}$ is a disturbance trajectory. 
Furthermore, we rely on the System Level Synthesis
(SLS)~\cite{wang2019_SLS} framework to obtain an explicit sub-Gaussian proxy and
the Wasserstein Lipschitz bounds in terms of the closed-loop maps from 
disturbances to state-control trajectories.

\subsection{Finite-horizon LTI specialization}
Consider a linear time-invariant (LTI) discrete-time dynamical system subject to additive disturbances
\begin{align}
    x_{k+1} &= A x_k + B u_k + w_k,  \quad k \in  \{0,\dots,T-1\},
    \label{eq:LTI-system}
\end{align}
where $x_k \in \mathbb{R}^{n_x}$ is the state, and $x_0$ is an initial state. 
Vector $u_k \in \mathbb{R}^{n_u}$ is the control input, and $w_k \in \mathbb{R}^{n_x}$ is an exogenous disturbance.
We consider finite-horizon linear causal controllers, which may be time-varying and 
depend on the state history up to the current time, i.e.
$$
u_k = \sum_{t=0}^k K_{k,t} x_t,  \quad k \in  \{0,\dots,T-1\}.
$$
Let $\mathcal{K}$ denote a class of finite-horizon linear causal controllers for system $(A,B)$,
where $A \in \Re^{n_x \times n_x}, B \in \Re^{nx \times n_u}$, 
and let $K \in \K$ be such a controller. Moreover,  let $\mathcal{W}$ 
denote the space of disturbance \emph{trajectories} over the horizon $T$, and $\bw \in \W$ 
 a particular disturbance trajectory.
Notice that, in this specialization, the data space $\Z$ becomes the disturbance-trajectory space $\W$.
The rollout loss function is now
$$\ell : \K \times \W \mapsto \Re_{+}.$$

\subsection{Finite-horizon lifted SLS responses}
Consider the discrete-time LTI system defined in~\eqref{eq:LTI-system}.
Define the stacked vectors
\begin{align*}
\mathbf{x}
&:=
\begin{bmatrix}
x_0^\top & x_1^\top & \cdots & x_T^\top
\end{bmatrix}^\top
\in \mathbb{R}^{(T+1)n_x},\\
\mathbf{u}
&:=
\begin{bmatrix}
u_0^\top & u_1^\top & \cdots & u_{T-1}^\top
\end{bmatrix}^\top
\in \mathbb{R}^{T n_u},\\
\mathbf{w}
&:=
\begin{bmatrix}
x_0^\top & w_0^\top & \cdots & w_{T-1}^\top
\end{bmatrix}^\top
\in \mathbb{R}^{(T+1)n_x}.
\end{align*}
Notice that we include the initial state in $\bw$ to ease the notation.
For a finite-horizon controller, the stacked state and input trajectories
are deterministic linear maps of the stacked disturbance vector:
\begin{equation}
\mathbf{x} = \Phi_x \mathbf{w}, \qquad \mathbf{u} = \Phi_u \mathbf{w},
\label{eq:sls_responses}
\end{equation}
where
$
\Phi_x \in \mathbb{R}^{(T+1)n_x \times (T+1)n_x},
\qquad
\Phi_u \in \mathbb{R}^{Tn_u \times (T+1)n_x}.
$
Both matrices $\Phi_x$ and $\Phi_u$ have a lower block-triangular structure 
due to the causality requirement.
We can write lifted responses in block-row form as
$$
\Phi_x =
\begin{bmatrix}
\Phi_x[0]\\
\Phi_x[1]\\
\vdots\\
\Phi_x[T]
\end{bmatrix},
\qquad
\Phi_u =
\begin{bmatrix}
\Phi_u[0]\\
\Phi_u[1]\\
\vdots\\
\Phi_u[T-1] 
\end{bmatrix},
$$
where
$
\Phi_x[k]\in\mathbb R^{n_x\times (T+1)n_x}, 
\qquad
\Phi_u[k]\in\mathbb R^{n_u\times (T+1)n_x}.
$

The finite-horizon achievability constraints are
\begin{equation}
\begin{split}
\Phi_x[0]
&=
\begin{bmatrix}
I_{n_x} & 0 & \cdots & 0
\end{bmatrix},
\\
\Phi_x[k+1]&=A\Phi_x[k]+B\Phi_u[k],
\,\, k \in  \{0,\dots,T-1\}
\label{eq:sls_achievability}
\end{split}
\end{equation}

These constraints define an affine feasible set over
\(\Phi \dfn (\Phi_x,\Phi_u)\), which can equivalently be written in compact
lifted form as
\begin{align*}
\mathcal{F}_{A,B}
=
\left\{ (\Phi_x,\Phi_u) : F_x \Phi_x + F_u \Phi_u = I \right\}.
\end{align*} where, $F_x$ and $F_u$ are given by
\begin{align*}
\overbrace{
\begin{bmatrix}
I      & 0      & 0      & \cdots & 0 \\
-A     & I      & 0      & \cdots & 0 \\
0      & -A     & I      & \ddots & \vdots \\
\vdots & \ddots & \ddots & \ddots & 0 \\
0      & \cdots & 0      & -A     & I
\end{bmatrix}}^{F_x},
\overbrace{
\begin{bmatrix}
0      & 0      & 0      & \cdots & 0 \\
-B     & 0      & 0      & \cdots & 0 \\
0      & -B     & 0      & \ddots & \vdots \\
\vdots & \ddots & \ddots & \ddots & 0 \\
0      & \cdots & 0      & -B     & 0
\end{bmatrix}}^{F_u}.
\end{align*} 
We refer the reader to~\cite{wang2019_SLS} for further details and derivations of these equalities.

\subsection{Subspace parameterization}
To obtain a finite-dimensional hypothesis space for PAC-Bayes, we
vectorize the lifted closed-loop response matrices. Let
\[
\operatorname{vec}(\Phi):=
\begin{bmatrix}
\operatorname{vec}(\Phi_x)\\
\operatorname{vec}(\Phi_u)
\end{bmatrix}.
\]
Since the finite-horizon SLS feasibility constraints are affine in
$(\Phi_x,\Phi_u)$, they can be written in vectorized form as
$
F\,\operatorname{vec}(\Phi)=b,
$
for a matrix $F$ and vector $b$ obtained from vectorizing the affine feasibility constraints. 
Hence, the feasible set is an affine
subspace in vectorized coordinates. Let $\Phi_0$ be any feasible baseline
lifted closed-loop response satisfying
$
F\,\operatorname{vec}(\Phi_0)=b,
$
and let $H$ be a matrix whose columns form a basis for the nullspace of $F$~\cite{golub2013matrix}.
Then every feasible lifted closed-loop response can be parameterized as
\[
\operatorname{vec}(\Phi)=\operatorname{vec}(\Phi_0)+H\theta,
\qquad \theta\in\Theta\subseteq\mathbb{R}^d.
\]
In this formulation, the learnable variable is $\theta \in \Theta$. Thus, the PAC-Bayesian hypothesis is now $\theta$ 
(as opposed to $K$ in previous sections), while the closed-loop response matrices become $\Phi_x(\theta)$ and $\Phi_u(\theta)$.
%

Furthermore, let $Q_c \in \mathbb{S}^{(T+1)n_x}_{+}$ and $R_c\in \mathbb{S}^{Tn_u}_{++}$ denote fixed performance weighting 
matrices, and define the stacked performance output
\begin{equation}
y(\theta) := \begin{bmatrix} Q_c^{1/2}  & 0 \\ 0 &  R_c^{1/2} \\ \end{bmatrix} \begin{bmatrix} \bx \\ \bu\end{bmatrix} = M(\theta) \bw,
\end{equation}
where
\begin{equation}
M(\theta):= \begin{bmatrix} Q_c^{1/2} \Phi_x(\theta) \\ R^{1/2}_c \Phi_u(\theta) \end{bmatrix}.
\label{eq:weighted_output}
\end{equation}

We will consider  $2$-norm type losses for simplicity of exposition, but
others, such as the $1$-norm, are readily incorporated. 
\begin{assumption} For any $\theta \in \Theta$ and $\bw \in \W$,
let the rollout loss be given by
\begin{equation}
\ell(\theta,\bw) := \|M(\theta)\bw\|.
\label{eq:l2_loss}
\end{equation}
\label{ass:l2_loss}
\end{assumption}
The empirical risk on a sample $S=\{\bw_i\}_{i=1}^n$ of i.i.d. trajectory-level 
disturbances is then
\begin{equation}
\empRisk_S(\theta)
:=
\frac{1}{n}\sum_{i=1}^n \|M(\theta)\bw_i\|,
\label{eq:empirical_risk_sls}
\end{equation}
and the population risk is 
\begin{equation}
R(\theta) := \E_{\bw \sim \genDstro} [ \ell(\theta, \bw)] =  \E_{\bw \sim \genDstro} \left[ \|M(\theta)\bw\| \right].
\label{eq:population_risk_sls}
\end{equation}
We now verify the conditions of Theorem~\ref{thm:PACBayes-W1} for two useful disturbance models: 
Gaussian disturbance trajectories and almost surely bounded ones.
Note once again that $\genDstro$ is  the disturbance-trajectory training distribution.

\subsection{Controller-dependent concentration and robustness certificates}
We start off with a  proposition on Gaussian trajectories.
\begin{proposition}[Gaussian trajectory disturbances]
\label{prop:gaussian_trajectory}
Assume
\begin{align*}
\bw \sim \mathcal{N}(\mu_w,\Sigma_w).
\end{align*}
For the loss  satisfying Assumption~\ref{ass:l2_loss}, the centered loss
\begin{align*}
\ell(\theta,\bw)-R(\theta)
\end{align*}
is $\sigma(\theta)$-sub-Gaussian with
\begin{align}
\sigma(\theta)
\dfn
\|M(\theta)\Sigma_w^{1/2}\|_{\mathrm{op}}.
 \label{eq:gaussian_sls_proxy}
\end{align}
\end{proposition}
\begin{proof}
    Loss is  $\ell(\theta,\bw) = \|M(\theta) \bw\|$ by Assumption~\ref{ass:l2_loss}, 
    with $\bw \sim \gaussian(\mu_w,\Sigma_w)$. 
    Reparametrize $\bw$ as
    $$\bw = \mu_w + \Sigma_w^{{1/2}}\xi$$ with $\xi \sim \gaussian(0,I)$.
    Define $\bar{\ell}_\theta(\xi) =\|M(\theta) (\mu_w +  \Sigma_w^{{1/2}}\xi) \|. $ 
    For any $\xi,\xi'$, it follows 
 $$|\bar{\ell}_\theta(\xi) - \bar{\ell}_\theta(\xi')|  
 \le\| M(\theta) \Sigma_w^{{1/2}}(\xi - \xi')\|$$ by the reverse triangle inequality.
 By using the definition of the operator norm, 
 we further have
 $
 \|M(\theta) \Sigma_w^{{1/2}}(\xi - \xi')\| \le \| M(\theta) \Sigma_w^{{1/2}} \|_\text{op} \|\xi - \xi'\|
 $
 which gives the Lipschitz bound $L(\theta) \dfn \| M(\theta) \Sigma_w^{{1/2}} \|_\text{op}$ for $\bar{\ell}_\theta$. 
 By the Gaussian concentration theorem~\cite[Theorem 5.5]{boucheron2013concentration},
 $\bar{\ell}_\theta - \E[\bar{\ell}_\theta]$ is $\sigma(\theta)$-sub-Gaussian, with $\sigma(\theta) = \| M(\theta) \Sigma_w^{{1/2}} \|_\text{op}$.
 Since $\ell(\theta, \bw) = \bar{\ell}_\theta(\xi)$, this completes the proof.
\end{proof}

\begin{proposition}[Bounded trajectory disturbances]
\label{prop:bounded_trajectory}
Assume that 
$$
\|\bw\| \le R 
$$
almost surely and the loss function satisfies Assumption~\ref{ass:l2_loss}.
Then the centered loss 
\begin{align*}
\ell(\theta,\bw)-R(\theta)
\end{align*}
is $\sigma(\theta)$-sub-Gaussian with
\begin{align}
\sigma(\theta)
\dfn
\frac{R}{2}\|M(\theta)\|_{\mathrm{op}}.
 \label{eq:bounded_sls_proxy}
\end{align}
\end{proposition}
\begin{proof}
For the operator norm and $R$ bounded norm of $\bw$, we have
$$
\|M(\theta)\bw\| \le \|M(\theta)\|_\text{op} \|\bw\| \le R\|M(\theta)\|_\text{op} 
$$
Since $l(\theta,\bw)\ge 0$ by construction, it holds that
$$
0 \le \ell(\theta,\bw) \le R\|M(\theta)\|_\text{op}.
$$
By Hoeffding's lemma~\cite{vershynin2018high} for bounded random variables, it follows that
the centered loss is $\sigma(\theta)$-sub-Gaussian, where
$$
\sigma(\theta)\dfn\frac{R}{2} \|M(\theta)\|_\text{op}.
$$
\end{proof}

\begin{proposition}[Wasserstein Lipschitz certificate]
\label{prop:sls_lipschitz}
Let the loss satisfy Assumption~\ref{ass:l2_loss}. Then it is $L(\theta)$-Lipschitz with a bound 
\begin{align}
L(\theta)
=
\|M(\theta)\|_{\mathrm{op}}.
\label{eq:lipschitz_sls} 
\end{align}
\end{proposition}
\begin{proof}
It holds that $|\ell(\theta,\bw)  - \ell(\theta,\bw')|  \le \| M(\theta)(\bw - \bw')\| \le \|M(\theta)\|_\text{op}\|\bw - \bw'\|$.
The first inequality follows from the reverse triangle inequality, while the second one follows from the definition of the operator norm.
$ \|M(\theta)\|_\text{op}$ is a Lipschitz constant  by definition.
\end{proof}

The quantities \eqref{eq:gaussian_sls_proxy}, \eqref{eq:bounded_sls_proxy}, and \eqref{eq:lipschitz_sls} 
are controlled by the same weighted closed-loop map $M(\theta)$, 
which is one of the main advantages of the SLS formulation in the present framework.

\begin{remark}[Interpretation]
Theorem~2 requires only a valid controller/hypothesis  sub-Gaussian proxy for the centered loss. Propositions~1 and~2 show that, for the finite-horizon SLS loss
$
 \ell(\theta,w)=\|M(\theta)w\|,
$
this requirement is satisfied in two example trajectory-level disturbance models. In the Gaussian case, the proxy is covariance-aware and takes the form
\begin{align}
\sigma(\theta)=\|M(\theta)\Sigma_w^{1/2}\|_{\mathrm{op}},
\end{align}
where the covariance matrix $\Sigma_w$ may be full. Thus, the Gaussian model allows for a correlation across the stacked disturbance trajectory, including temporal correlations, cross-coordinate correlations, etc. 
In the bounded case, the proxy becomes
\begin{align}
\sigma(\theta)=\|M(\theta)(R/2)I\|_{\mathrm{op}}
=\frac{R}{2}\|M(\theta)\|_{\mathrm{op}}.
\end{align}

Hence, both certificates share the same template as they result from 
combining a concentration inequality with a Lipschitz bound for the disturbance to loss map.
The uncertainty enters through a scaling factor multiplying the closed-loop map. 
In the Gaussian case, this factor is $\Sigma_w^{1/2}$, while in the bounded case it is the isotropic scaling $(R/2)I$. 
\end{remark}


\subsection{Robust PAC-Bayes objective in SLS form}
Combining the PAC-Bayes square-root bound given by Theorem~\ref{thm:PACBayes-W1} with SLS 
hypothesis $\theta$ yields the final learning algorithm over the posterior $Q$.
Using the explicit proxies $\sigma(\theta)$ 
given by \eqref{eq:gaussian_sls_proxy} or \eqref{eq:bounded_sls_proxy}, 
we state the posterior optimization problem as 
\begin{equation}
\minimize_{Q \ll P} \,\,\mathbb{E}_{\theta\sim Q}\!\left[
\widehat R_S(\theta) + \rho \|M(\theta)\|_{\mathrm{op}}
\right]
 + \cmplx(Q,P,\sigma),
\label{eq:optimization_problem}
\end{equation}
with the complexity term $\cmplx(Q,P,\sigma)$ being equal to 
\begin{align}
\sqrt{
\frac{
2\,\mathbb{E}_{\theta\sim Q}\!\left[\sigma(\theta)^2\right]\,
\bigl(\mathrm{KL}(Q\|P)+\log(n/\delta)\bigr)
}{
n-1
}
}.
\end{align}
\begin{remark}[Computation] 
The optimization problem~\eqref{eq:optimization_problem} can be hard to solve in its general form. One common restriction is to
consider posteriors and priors that belong to the Gaussian family. Then, the KL term is available in closed form. 
For Gaussian distributions $\mathcal{N}(\mu,\Sigma)$, the empirical-risk and certificate terms can be estimated by
Monte Carlo methods using the reparameterization trick
$
\theta = \mu + L\varepsilon,
$
$
\varepsilon \sim \mathcal{N}(0,I),$
where $L$ is the square root of the covariance matrix. This yields a differentiable 
finite-dimensional optimization problem that is readily exploited by modern automatic differentiation (AD) methods. 
Otherwise, for the general Gibbs posterior, we can employ more advanced methods such as~\cite{welling2011bayesian,liu2016stein} similar 
to what was done in~\cite{BeroujeniPACSNOC}.
\end{remark}
Next, we follow up with a numerical example.

%% file: text/numerical_double_integrator.tex
\section{Numerical Example}
\label{sec:numerical}
In what follows, we use Julia~\cite{bezanson2017julia} to implement the numerical example and Zygote~\cite{Zygote.jl-2018} as the AD backend. 
The optimization problem is formulated in JuMP~\cite{DunningHuchetteLubin2017_JUMP}.

\subsection{Double integrator}
We consider a finite-horizon control problem for the discrete-time linear system~\eqref{eq:LTI-system}
with
\begin{align}
A = \begin{bmatrix} 1.0 & 0.1 \\ 0.0 & 1.0 \end{bmatrix},
\,
B = \begin{bmatrix} 0.0 \\ 1.0 \end{bmatrix} 
\end{align}
The state and input penalties are
$$
\widehat{Q}_c = \operatorname{diag}(1.0,\,0.1),
\qquad
\widehat{R}_c = \operatorname{diag}(0.01),
$$
and $Q_c,R_c$ are constructed as  
\begin{align*}
Q_c := \mathrm{blkdiag}(\underbrace{\widehat{Q}_c,\dots,\widehat{Q}_c}_{T+1}),
\,\,
R_c := \mathrm{blkdiag}(\underbrace{\widehat{R}_c,\dots,\widehat{R}_c}_{T}).
\end{align*}
for the control horizon $T=10$.
For the stacked trajectory, the loss is the weighted Euclidean norm
$$
\ell(\bx,\bu)
= \left(\sum_{t=0}^{T} x_t^\top \widehat{Q}_c  x_t + \sum_{t=0}^{T-1} u_t^\top \widehat{R}_c u_t \right)^{1/2}.
$$
In the SLS specialization, this trajectory cost is precisely $\ell(\theta,\bw) = \|M(\theta) \bw\|$.
We construct the corresponding finite-horizon SLS parameterization $\Phi = \Phi_0 + H\theta$ and optimize an isotropic Gaussian posterior
$$
Q= \mathcal{N}(\mu, \sigma^2 I)
$$
over the SLS coordinates $\theta$. The training set consists of $n$ i.i.d. disturbance trajectories drawn from a zero-mean Gaussian with covariance
$$
\Sigma_w = (0.02)^2 I, 
$$
which, in this example, corresponds to disturbance vectors of dimension $N_w = (T+1)n_x = 22$. The posterior is fitted with the  solver using max $150$ L-BFGS iterations, at a  confidence level $\delta=0.05$, a prior standard deviation $\sigma_{\mathrm{prior}}=1.0$, and $24$ Monte Carlo samples to estimate posterior expectations. 
We report the bound decomposition on the training set and the comparison of the vanilla 
PAC-Bayes and the distributionally robust version developed in this work. 
The results are depicted in Figures~\ref{fig:shift_baseline_cost distribution} and~\ref{fig:robust_vs_nonrobust_pac_bayes}.
 \begin{figure}
    \includegraphics[scale=0.36]{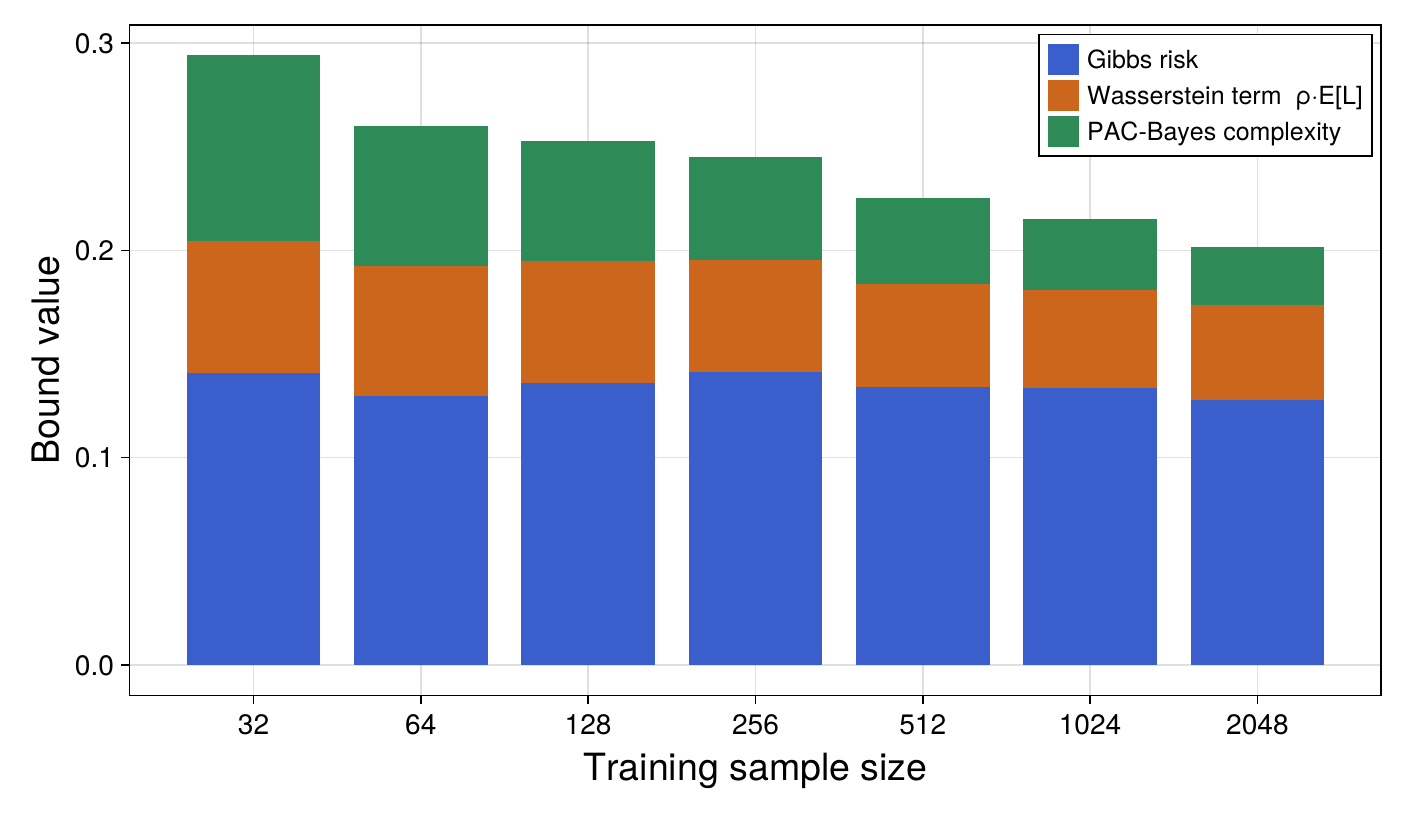}
    \caption{Decomposition of the  robust PAC-Bayes certificate cost for various training sample sizes. 
    Gibbs risk is defined as empirical cost averaged over the optimized posterior $Q$. As the PAC-Bayes complexity depends on the controller performance, the resulting cost is self-balancing without the need to resort to tricks such as saturation.}
        \label{fig:shift_baseline_cost distribution}
\end{figure}

\begin{figure} 
    \includegraphics[scale=0.34]{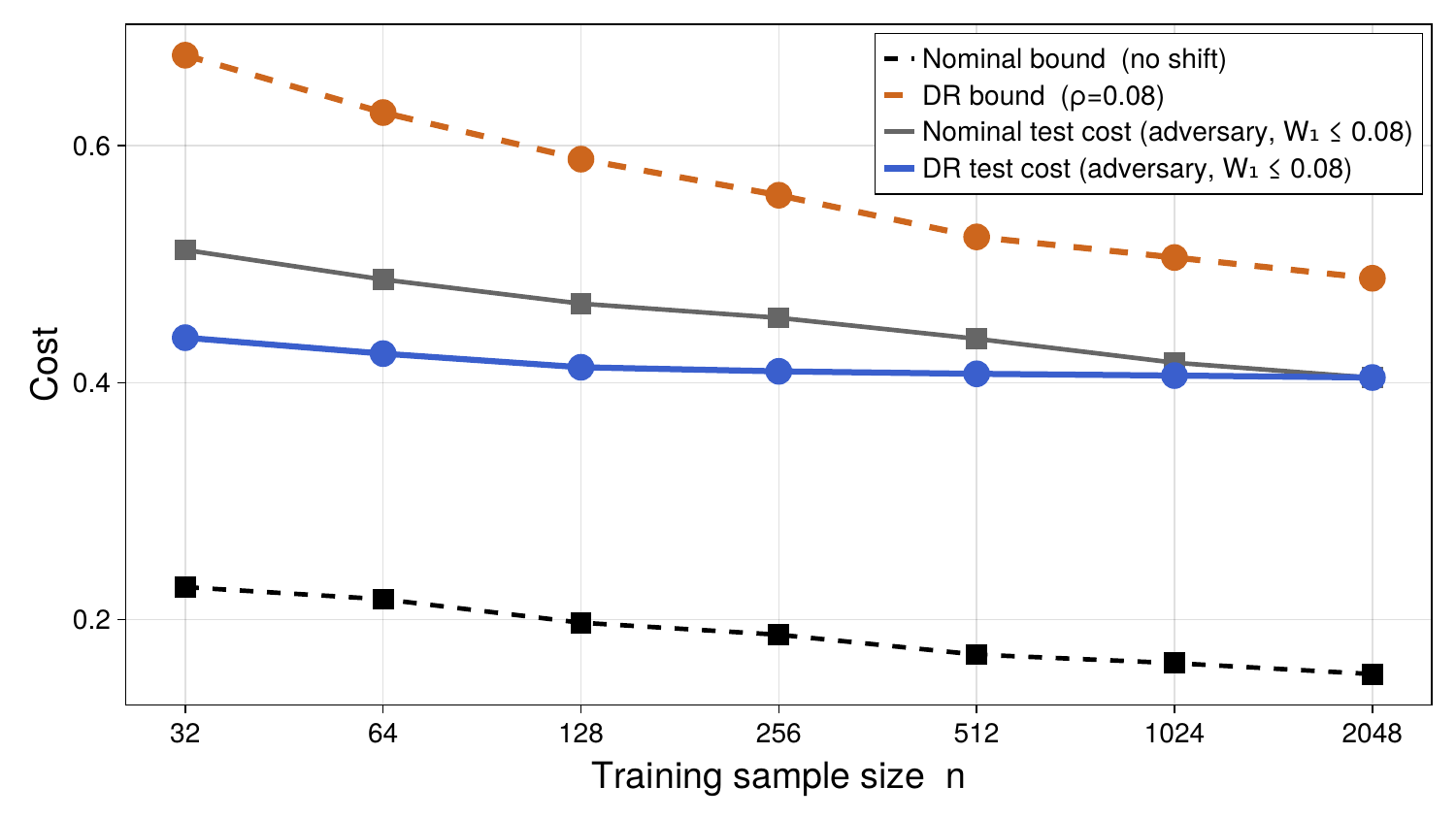}
    \caption{Comparison of vanilla PAC-Bayes ($\rho = 0$) vs robust PAC-Bayes ($\rho = 0.08$). The vanilla PAC-Bayes bound (black dashed) does not properly upper-bound the test risk (gray solid)  due to unmodeled distribution shift, while the robust bound (orange dashed) provides a valid bound on the test risk (blue solid).  We can also observe that DRO PAC-Bayes results in a controller that has lower empirical cost in addition to correct theoretical guarantee. Both empirical test costs are validated against the same deployment distribution shift, suggesting the twofold benefit of our DRO PAC-Bayes approach. 
    We empirically observe the same behavior over different $\rho$ values. We report mean values for test cost.
    } 
    \label{fig:robust_vs_nonrobust_pac_bayes}
\end{figure}

 Looking at Figure~\ref{fig:shift_baseline_cost distribution}, we can see that the PAC-Bayes complexity term diminishes with more data, as expected. However, the optimization procedure still has to balance the contribution from the Wasserstein penalty and the empirical risk, 
 both averaged over the same posterior. 

 Figure~\ref{fig:robust_vs_nonrobust_pac_bayes} shows the effectiveness of our method in the presence of distribution shifts. We can observe that vanilla PAC-Bayes, which stands for the PAC-Bayes term without the Wasserstein part, cannot account for environmental shifts properly. It can clearly be seen from the figure that the reported bound is violated for all sizes of the dataset $n$. On the other hand,
 our robustified method provides a correct upper-bound on the realized empirical risk. We also remark that our method
 performs better on distribution shifted test data than the vanilla one. Both methods have been tested on the same adversarial distribution shift.
 Note that the distribution shift was selected such that it stays inside the certified radius while moving the mean of the training data distribution. 